\documentclass[letterpaper, 10 pt, conference]{ieeeconf}  
\IEEEoverridecommandlockouts                              

\usepackage{verbatim}
\usepackage{amssymb}
\usepackage{mathtools}
\usepackage{amsmath}
\usepackage{commath}
\usepackage{nth}
\usepackage{cite}
\usepackage{cleveref}

\usepackage[vlined,ruled]{algorithm2e}

\usepackage{amsmath,upgreek,graphicx,epsfig,color,amsfonts}

\graphicspath{{./fig/}}

\def\prob{\mathbb{P}}
\def\expt{\mathbb{E}\,}
\def\real{\mathbb{R}}
\def\integer{\mathbb{Z}}

\def\indicator{\mathbf{1}}
\newcommand{\until}[1]{\{1,\dots, #1\}}

\newcommand{\supscr}[2]{#1^{\textup{#2}}}
\newcommand{\setdef}[2]{\{#1 \; | \; #2\}}
\newcommand{\seqdef}[2]{\{#1\}_{#2}}

\newcommand{\ceil}[1]{\left\lceil #1 \right\rceil}
\newcommand{\floor}[1]{\left\lfloor #1 \right\rfloor}

\newcommand\oprocendsymbol{\hbox{$\square$}}
\newcommand\oprocend{\relax\ifmmode\else\unskip\hfill\fi\oprocendsymbol}

\def \bs {\boldsymbol}

\DeclareMathOperator{\Var}{Var}
\DeclareMathOperator{\sat}{sat}
\DeclareMathOperator{\sign}{sign}

\newtheorem{theorem}{Theorem}

\newtheorem{lemma}[theorem]{Lemma}
\newtheorem{corollary}[theorem]{Corollary}

\newtheorem{remark}{Remark}

\newtheorem{assumption}{Assumption}

\renewcommand{\theenumi}{(\roman{enumi}}

\renewcommand{\baselinestretch}{0.99}

\title{Minimax Policy for Heavy-tailed Bandits 
	\thanks{This work has been supported by NSF Award IIS-1734272.}
}

\author{Lai Wei \hspace{1in} Vaibhav Srivastava
	\thanks{L. Wei and V. Srivastava are with the Department of Electrical and Computer Engineering. Michigan State University, East Lansing, MI 48823 USA.
		{\tt\small e-mail: weilai1@msu.edu; e-mail: vaibhav@egr.msu.edu }}%
}

\begin{document}

\maketitle
\thispagestyle{empty}
\pagestyle{empty}

\begin{abstract}
We study the stochastic Multi-Armed Bandit (MAB) problem under worst-case regret and heavy-tailed reward distribution. We modify the minimax policy MOSS for the sub-Gaussian reward distribution by using saturated empirical mean to design a new algorithm called Robust MOSS. We show that if the moment of order $1+\epsilon$ for the reward distribution exists, then the refined strategy has a worst-case regret matching the lower bound while maintaining a {distribution-dependent} logarithm regret.
\end{abstract}


\section{Introduction}\label{sec:introduction}
The dilemma of exploration versus exploitation is common in scenarios involving decision-making
in unknown environments. In these contexts, exploration means learning the environment while exploitation means taking empirically {computed} best actions. When finite time performance is concerned, {i.e., scenarios in which} one cannot learn indefinitely, ensuring a good balance of exploration and exploitation is the key to a good performance. {MAB and its variations are prototypical models for these problems, and they are widely used in many areas such as network routing, recommendation systems and resource allocation; see~\cite[Chapter 1]{lattimore2020bandit}.}

The stochastic MAB problem was originally proposed by Robbins~\cite{robbins1952}. In this problem, at each time, an agent chooses an arm from a set of $K$ arms and receives the associated reward. The reward at each arm is a stationary random variable with an unknown mean.  The objective is to design a policy that maximizes the expected cumulative reward or equivalently  minimizes the \emph{expected cumulative regret}, defined by the expected cumulative difference between the maximum mean reward and the reward obtained using the policy. 

The worst-case regret is defined by the supremum of the expected cumulative regret computed over a class of reward distributions, e.g., sub-Gaussian distributions, or distributions with bounded support. The \emph{minimax regret} is defined as the minimum worst-case regret, where the minimum is computed over all the policies.
By construction, the worst-case regret  uses minimal information about the underlying distribution and the associated regret bounds are called \emph{distribution-free bounds}. In contrast, the standard regret bounds depend on the difference between the mean rewards from the optimal and suboptimal arms, and the corresponding bounds are referred as \emph{distribution-dependent bounds}.   	




In their seminal work, Lai and Robbins~\cite{TLL-HR:85} establish that the expected cumulative regret admits an asymptotic {distribution-dependent} lower bound that is a logarithmic function of the time-horizon $T$. Here, asymptotic refers to the limit $T \to +\infty$. 
They also propose a general method of constructing Upper Confidence Bound (UCB) based policies that attain the lower bound asymptotically. By assuming rewards to be bounded or more generally sub-Gaussian, several subsequent works design simpler algorithms with finite time performance guarantees, e.g., the UCB1 algorithm by Auer et al.~\cite{PA-NCB-PF:02}. By using Kullback-Leibler(KL) divergence based upper confidence bounds, Garivier and Capp{\'e}~\cite{AG-OC:11} designed KL-UCB, which is proved to have efficient finite time performance as well as asymptotic optimality.


In the worst-case setting,  the lower and upper bounds are distribution-free. Assuming the rewards are bounded, Audibert and Bubeck~\cite{MOSS} establish a $\Omega(\sqrt{KT})$ lower bound on the minimax regret. They also studied a modified UCB algorithm called Minimax Optimal Strategy in the Stochastic case (MOSS) and proved that it achieves an order-optimal worst-case regret while maintaining a logarithm distribution-dependent regret. Degenne and Perchet~\cite{moss_anytime} extend MOSS to an any-time version called MOSS-anytime.




The rewards being bounded or sub-Gaussian is a common assumption that gives sample mean an exponential convergence and simplifies the MAB problem. However in many applications, such as social networks~\cite{albert2002statistical} and financial markets~\cite{vidyasagar2010law}, the rewards are heavy-tailed. For the standard stochastic MAB problem, Bubeck et al.~\cite{bubeck2013bandits} relax the sub-Gaussian assumption by only assuming the rewards to have finite moments of order $ 1+\epsilon$ for some $\epsilon \in (0, 1]$. 
They present the robust UCB algorithm and show that it attains an upper bound on the cumulative regret that is within a constant factor of the {distribution-depend} lower bound in the heavy-tailed setting. {However, the solutions provided in~\cite{bubeck2013bandits} are not able to provably achieve an order optimal worst-case regret. Specifically, the factor of optimality is a poly-logarithmic function of time-horizon. 
}

{In this paper, we study the minimax heavy tail bandit problem in which reward distributions admit moments of  order $1+\epsilon$, with $\epsilon>0$. 
	We propose and analyze Robust MOSS algorithm to show that it achieves worst-case regret matching with the lower bound while maintaining a distribution-dependent logarithm regret. To the best of our knowledge, Robust MOSS is the first algorithm to achieve order optimal worst-case regret for heavy-tailed bandits. Our results build on techniques in~\cite{MOSS} and~\cite{bubeck2013bandits}, and augment them with new analysis based on maximal Bennett inequalities.} 


The remaining paper is organized as follows. We describe the minimax heavy-tailed MAB problem and present some background material in Section~\ref{sec:background}. We present and analyze the Robust MOSS algorithm in Sections~\ref{sec:algo} and~\ref{sec:analysis}, respectively, and numerically compare it with the state of the art in Section~\ref{sec:simulation}. We conclude in Section~\ref{sec:conclusions}.

\section{Background \& Problem Description}\label{sec:background}

\subsection{Stochastic MAB Problem}
In a stochastic MAB problem, an agent chooses an arm $\varphi_t$ from the set of $K$ arms $\until{K}$ at each time $t \in \until{T}$ and receives the associated reward. The reward at each arm $k$ is drawn from an unknown distribution $f_k$ with unknown mean $\mu_k$. Let the maximum mean reward among all arms be $\mu^*$. 
We use  $\Delta_k =\mu^*-\mu_k$ to measure the suboptimality of arm $k$. The objective is to maximize the expected cumulative reward or equivalently to minimize the \emph{expected cumulative regret} defined by
\[R_T : = \expt \Bigg[\sum_{t=1}^T   \left(\mu^* - X_{\varphi_t} \right)\Bigg] = \expt \Bigg[\sum_{t=1}^T \Delta_{\varphi_t} \Bigg],  \]
which is the difference between the expected cumulative reward obtained by selecting the arm with the maximum mean reward $\mu^*$ and selecting arms $\varphi_1, \ldots, \varphi_T$. 

The expected cumulative regret $R_T$ is implicitly defined for a fixed distribution of rewards from each arm $\{f_1, \ldots, f_K\}$. The worst-case regret is the expected cumulative regret for the worst possible choice of reward distributions. In particular, 
\[
\supscr{R_T}{worst} = \sup_{\{f_1, \ldots, f_K\}} R_T. 
\]
The regret associated with the policy that minimizes the above worst-case regret is called \emph{minimax regret}. 

\subsection{Problem Description: Heavy-tailed Stochastic MAB}

In this paper, we study the heavy-tailed stochastic MAB problem, which is the stochastic MAB problem with following assumptions.

\begin{assumption}\label{ass1}
	Let $X$ be a random reward drawn from any arm $k \in \until{K}$. There exists a constant $u \in \real_{>0}$ such that $\expt \big[\abs{X}^{1+\epsilon}\big] \leq u^{1+\epsilon}$ for some $\epsilon \in (0,1]$.
\end{assumption}
\begin{assumption}\label{ass2}
	Parameters $T$, $K$, $u$ and $\epsilon$ are known.	
\end{assumption}

\subsection{MOSS Algorithm for Worst-Case Regret}
We now present the MOSS algorithm proposed in~\cite{MOSS}. The MOSS algorithm is designed for the stochastic MAB problem with bounded rewards and in this paper, we extend it to design Robust MOSS algorithm for heavy-tailed bandits. 

Suppose that  arm $k$ is sampled $n_k(t)$ times until time $t-1$, and $ \bar \mu^k_{n_k(t)}$ is the associated empirical mean, then, at time $t$, MOSS picks the arm that maximizes the following UCB
\[g^k_{n_k(t)} = \bar \mu^k_{n_k(t)} + \sqrt{\frac{ \max  \left(\ln \left(\frac{T}{K n_k(t)}\right), 0\right)}{n_k(t)}}.\]

If the rewards from the arms have bounded support $[0,1]$, then the worst-case regret for MOSS satisfies $\supscr{R_T}{worst} \leq 49\sqrt{KT}$, which is order optimal~\cite{MOSS}. Meanwhile, MOSS maintains a logarithm distribution-dependent regret bound.

\subsection{A Lower Bound for Heavy-tailed Minimax Regret}
We now present the lower bound on the minimax regret for the heavy tailed bandit problem derived in~\cite{bubeck2013bandits}. 
\begin{theorem}[{\cite[Th. 2]{bubeck2013bandits}}] For any fixed time horizon $T$ and the stochastic MAB problem under Assumptions~\ref{ass1} and~\ref{ass2} with $u=1$,
	\[\supscr{R_T}{worst} \geq 0.01 K^{\frac{\epsilon}{1+\epsilon}} T^{\frac{1}{1+\epsilon}}. \]
\end{theorem}
\begin{remark}
	Since $R_T$ scales with $u$, the lower bound for heavy tail bandit is $\Omega \big(u K^{\frac{\epsilon}{1+\epsilon}} T^{\frac{1}{1+\epsilon}}\big)$. This lower bound also indicates that within a finite horizon $T$, it is almost impossible to differentiate the optimal arm from arm $k$, if $\Delta_k \in O \big(u (K/T)^{\frac{\epsilon}{1+\epsilon}} \big)$. {As a special case, rewards with bounded support $[0,1]$ correspond to $\epsilon=1$ and $u=1$. Then, the lower bound $\Omega(\sqrt{KT})$ matches with the regret upper bound achieved by MOSS.}
\end{remark}

\section{A Robust Minimax Policy}\label{sec:algo}
To deal with the heavy-tailed reward distribution, we replace the empirical mean  with a saturated empirical mean. Although saturated empirical mean is a biased estimator, it has better convergence properties. We construct a novel UCB index to evaluate the arms, and at each time slot the arm with the maximum UCB index is picked.

\subsection{Robust MOSS}
In Robust MOSS, we consider a robust mean estimator called saturated empirical mean which is formally defined in the following subsection. Let $n_k(t)$ be the number of times that arm $k$ has been selected until time $t-1$. At time $t$, let $\hat \mu_{n_k(t)}^k$ be the saturated empirical mean reward computed from the $n_k(t)$ samples at arm $k$. Robust MOSS initializes by selecting each arm once and subsequently, at each time $t$, selects the arm that maximizes the following UCB
\[
g^k_{n_k(t)} = \hat \mu^k_{n_k(t)} + (1+\eta)c_{n_k(t)}, 
\]
where $\eta >0$ is an appropriate constant, $c_{n_k(t)} = u \times \big[\phi(n_k(t))\big]^{\frac{\epsilon}{1+\epsilon}}$ and
\[
\phi(n) = \frac{\ln_+ \big(\frac{T}{K n}\big)}{n},
\]
where $\ln_+(x) := \max (\ln x, 1)$. {Note that both $\phi(n)$ and $c_n$ are monotonically decreasing in $n$.}












\subsection{Saturated Empirical Mean}
The robust saturated empirical mean is similar to the truncated empirical mean used in~\cite{bubeck2013bandits}, which is employed to extend UCB1 to achieve logarithm {distribution-dependent} regret for the heavy-tailed MAB problem. Let $\seqdef{X_i}{i\in \until{m}}$ be a sequence of i.i.d. random variables with mean $\mu$ and $\expt \big[\abs{X_i}^{1+\epsilon}\big] \leq u^{1+\epsilon}$, where $u>0$. Pick $a >1$ and let $h(m) = a^{\floor{\log_a \left(m\right)}+1}$ {such that $h(m) \geq m$}. Define the saturation point $B_m$ by
\[ B_m := u \times \big[\phi\big(h(m)\big)\big]^{-\frac{1}{1+\epsilon}}.\]
Then, the saturated empirical mean estimator is defined by
\begin{equation}\label{def: sat_mean}
\hat \mu_m  := \frac{1}{m} \sum_{i=1}^m \sat (X_i,B_m),
\end{equation}
where $\sat (X_i,B_m) := \sign(X_i)\min \big\{\abs{X_i}, B_m \big\}.$

Define $d_i: = \sat(X_i,B_m)-\expt [\sat(X_i,B_m)]$. The following lemma examines the estimator bias and provides an upper bound on the error of saturated empirical mean.

\begin{lemma}[Error of saturated empirical mean]\label{lemma: estimator_error}
	For an i.i.d. sequence of random variables $\seqdef{X_i}{i\in\until{m}}$ such that $\expt[X_i] =\mu $ and $\expt \big[X_i^{1+\epsilon}\big] \leq u^{1+\epsilon}$, the saturated empirical mean~\eqref{def: sat_mean} satisfies
	\[
	\Bigg | {\hat \mu_m - \mu - \frac{1}{m} \sum_{i=1}^{m} d_i } \Bigg | \leq \frac{u^{1+\epsilon}}{B_m^\epsilon}. 
	\]
\end{lemma}

\begin{proof}
	Since $\mu = \expt \Big[ X_i \big(\bs{1}_{ \left\{ \abs{X_i} \leq B_m\right\} } + \bs{1}_{ \left\{ \abs{X_i} > B_m\right\} }\big)\Big]$, the error of estimator $\hat \mu_m$ satisfies
	\begin{align*}
	\hat \mu_m - \mu = & \frac{1}{m} \sum_{i=1}^{m} \left(\sat(X_i,B_m) -\mu\right) \\
	= & \frac{1}{m} \sum_{i=1}^{m} d_i + \frac{1}{m}\sum_{i=1}^{m} \left(\expt [\sat(X_i,B_m)] - \mu\right),
	\end{align*}
	where the second term is the bias of $\hat \mu_m $. We now compute an upper bound on the bias. 
	\begin{align*}
	\abs{\expt [\sat(X_i,B_m)] - \mu} &\leq \expt \left[ \abs{X_i} \bs{1}_{ \left\{ \abs{X_i} > B_m\right\} } \right]\\
	&\leq \expt \left[ \frac{\abs{X_i}^{1+\epsilon}}{(B_m)^\epsilon} \right] \leq \frac{u^{1+\epsilon}}{(B_m)^{\epsilon}}, 
	\end{align*}
	which concludes the proof.
\end{proof}

We now establish properties of $d_i$.
\begin{lemma}[Properties of $d_i$]\label{prop: d_i}
	For any $i\in \until{m}$, $d_i$ satisfies (i) $\abs{d_i}\leq 2 B_m$ (ii) $\expt[d_i^2] \leq u^{1+\epsilon}B_m^{1-\epsilon}$.
\end{lemma}
\begin{proof}
	Property (i) follows immediately from definition of $d_i$, and property (ii) follows from
	
	\medskip 
	
	$\qquad \displaystyle \expt[d_i^2] \leq \expt\big[\sat^2(X_i,B_m)\big] \leq  \expt\big[\abs{X_i}^{1+\epsilon}B_m^{1-\epsilon}\big].$
\end{proof}


\section{Analysis of Robust MOSS}\label{sec:analysis}
In this section, we analyze Robust MOSS to provide both distribution-free and distribution-dependent regret bounds.


\subsection{Properties of Saturated Empirical Mean Estimator}
To derive the concentration property of saturated empirical mean, we use a maximal Bennett type inequality as shown in Lemma~\ref{max_inq_b}. 

\begin{lemma}[Maximal Bennett's inequality~{\cite{fan2012hoeffding}}] \label{max_inq_b}
	Let $\seqdef{X_i}{i\in \until{n}}$ be a sequence of bounded random variables with support $[-B,B]$, where $B\geq 0$. Suppose that $\expt[X_i |X_{1},\ldots,X_{i-1}] = \mu_i$ and $\Var[X_i|X_{1},\ldots,X_{i-1}] \leq v$. Let $S_m = \sum_{i=1}^{m} (X_i -\mu_i) $ for any $m\in \until{n}$. Then, for any $\delta \geq 0$
	\begin{align*}
	&\prob\left( \exists {m \in \until{n}}:  S_m \geq \delta \right) \leq \exp \left( -\frac{\delta}{B}\psi \left (\frac{B\delta}{n v} \right) \right), \\
	&\prob\left(\exists {m \in \until{n}}: S_m \leq -\delta\right) \leq \exp \left(-\frac{\delta}{B}\psi \left (\frac{B\delta}{n v} \right)\right),
	\end{align*}
	where $\psi(x) =  (1+1/x)\ln(1+x)-1 $. 
\end{lemma}
\begin{remark}
	For $x\in (0,\infty)$, function $\psi(x)$ is monotonically increasing in $x$.
\end{remark}

Now, we establish an upper bound on the probability that the UCB underestimates the mean at arm $k$ by an amount $x$.
\begin{lemma}\label{lemma: suff_sample}
	For any arm $k\in \until{K}$ and any $t \in \left\{K+1,\ldots,T\right\}$ and $x > 0$, if $\eta\psi(2\eta /a) \geq 2a$, the probability of event $\big \{ g^k_{n_k(t)}  \leq \mu_k - x \big \}$ is no greater than
	\[\frac{K}{T} \frac{a}{\ln(a)} \Gamma\left(\frac{1}{\epsilon}+2\right)  \left(\frac{\psi \left ( 2\eta /a \right)}{2a} \frac{x}{u}\right)^{-\frac{1+\epsilon}{\epsilon}} .\]
\end{lemma}
\begin{proof}
	It follows from Lemma~\ref{lemma: estimator_error} that
	\begin{align*}
	&\prob \left(g^k_{n_k(t)} \leq \mu_k - x \right) \\
	\leq& \prob \left(\exists m \in \until{T}:  \hat{\mu}^k_m + (1+\eta)c_m  \leq \mu_k - x \right)\\
	\leq& \prob \bigg(\exists m \in \until{T}:  \sum_{i=1}^{m} \frac{d_i^k}{m} \leq \frac{u^{1+\epsilon}}{B_m^\epsilon}  - (1+\eta) c_m -x \bigg) 	\\
	\leq& \prob \bigg(\exists m \in \until{T}:  \frac{1}{m} \sum_{i=1}^{m} d_i^k \leq -x - \eta c_m \bigg),
	\end{align*}
	where $d_i^k$ is defined similarly to $d_i$ for i.i.d. reward sequence at arm $k$ and the last inequality is due to
	\begin{equation}
	\frac{u^{1+\epsilon}}{B_m^\epsilon} = u \big[\phi\big(h(m)\big)\big]^{\frac{\epsilon}{1+\epsilon}}  \leq u \big[\phi(m)\big]^{\frac{\epsilon}{1+\epsilon}} = c_m. \label{ineq: bc}
	\end{equation}
	Recall $a>1$. We apply a peeling argument~\cite[Sec 2.2]{bubeck2010bandits} with geometric grid $ a^s \leq m < a^{s+1}$ over time interval $\until{T}$. Since $c_m$ is monotonically decreasing with $m$,
	\begin{align*}
	& \prob \bigg(\exists m \in \until{T}:  \frac{1}{m} \sum_{i=1}^{m} d_i^k \leq -x - \eta c_m \bigg)\\
	\leq &\sum_{s\geq 0} \prob\bigg(\exists m \in [a^s, a^{s+1}) : \sum_{i=1}^{m} d_i^k \leq -a^s \left(x + \eta c_{a^{s+1}} \right) \bigg).	
	\end{align*}
	Also notice that $B_m = B_{a^s} $ for all $m \in [a^s, a^{s+1})$. Then with properties in Lemma~\ref{prop: d_i}, we apply Lemma~\ref{max_inq_b} to get	
	\begin{align}
	&\sum_{s\geq 0} \prob\bigg(\exists m \in [a^s, a^{s+1}) : \sum_{i=1}^{m} d_i^k \leq -a^s \left(x + \eta c_{a^{s+1}} \right) \bigg) \nonumber\\
	\leq &\sum_{s\geq 0} \exp \left( -\frac{a^{s} \left(x+ \eta c_{a^{s+1}}\right)}{2 B_{a^{s}}}\psi \left (\frac{ 2 B_{a^{s}} \left(x + \eta c_{a^{s+1}}\right)}{a u^{1+\epsilon}B_{a^{s}}^{1-\epsilon}} \right) \right)\nonumber\\
	&\left(\text{since } \psi(x) \text{ is monotonically increasing}\right)\nonumber\\
	\leq & \sum_{s\geq 0} \exp \left( -\frac{a^{s} \left(x+ \eta c_{a^{s+1}}\right)}{2 B_{a^{s}}}\psi \left (\frac{ 2 \eta B_{a^{s}}^\epsilon   c_{a^{s+1}}}{a u^{1+\epsilon}} \right) \right) \nonumber\\
	&\text{\big(plug in $c_{a^{s+1}}$, $B_{a^{s}}$ and use $h(a^s)=a^{s+1}$)} \nonumber \\
	=& \sum_{s\geq 1} \exp \left( -a^{s}\left(\frac{ x}{B_{a^{s-1}}} + \eta \phi(a^s)\right) \frac{\psi \left ( 2\eta /a \right)}{2a} \right)  \nonumber \\
	&\left( {{\text{plug in } \phi(a^s) \text{ and use }} \eta\psi(2\eta/a)\geq 2a,  \ln_+ (y) \geq \ln (y)}  \right)\nonumber\\
	\leq &  \sum_{s\geq 1}  \exp \left( -a^{s} \frac{ x}{B_{a^{s-1}}} \frac{\psi \left ( 2\eta / a \right)}{2a} \right) \frac{K}{T} a^s. \label{sum:1}
	\end{align}
	Let $b = {x\psi \left ( 2\eta / a \right)}/ (2au)$. 
	{Since $B_{a^{s-1}} \leq ua^{\frac{s}{1+\epsilon}}$, we have}
	\begin{align*}
	\eqref{sum:1}\leq & \frac{K}{T} \sum_{s\geq 1} a^s \exp \left( -b a^{\frac{\epsilon s}{1+\epsilon}} \right) \\
	\leq & \frac{K}{T} \int_{1}^{+\infty} a^y \exp\big(- b a^{\frac{(y-1)\epsilon}{1+\epsilon}}\big) dy  \\
	= & \frac{K}{T} a \int_{0}^{+\infty} a^y \exp\big(-b a^{\frac{y\epsilon}{1+\epsilon}}\big) dy \\
	&\left(\text{where we set } z=b a^{\frac{y\epsilon}{1+\epsilon}} \right)\\
	= & \frac{K}{T} \frac{a}{\ln(a)}\frac{1+\epsilon}{\epsilon} {b^{-\frac{1+\epsilon}{\epsilon}}} \int_{b}^{+\infty} z^{\frac{1+\epsilon}{\epsilon}-1} \exp\big(- z  \big) dz \\
	\leq & \frac{K}{T} \frac{a}{\ln(a)} \Gamma\left(\frac{1}{\epsilon}+2\right)b^{-\frac{1+\epsilon}{\epsilon}},
	\end{align*}
	which concludes the proof. 
\end{proof}

The following is a straightforward corollary of Lemma~\ref{lemma: suff_sample}.
\begin{corollary}\label{corollary: suff_sample}
	For any arm $k\in \until{K}$ and any $t \in \left\{K+1,\ldots,T\right\}$ and $x > 0$, if $\eta\psi(2\eta/a) \geq 2a$, the probability of event $\big\{g^k_{n_k(t)}-2(1+\eta)c_{n_k(t)}\geq  \mu_k + x\}$ shares the same bound in Lemma~\ref{lemma: suff_sample}.
\end{corollary}

\subsection{Distribution-free Regret Bound}
The distribution-free upper bound for Robust MOSS, which is the main result for the paper, is presented in this section. We show that the algorithm achieves order optimal worst-case regret. 

\begin{theorem}\label{theorem:minimax bound}
	For the heavy-tailed stochastic MAB problem with $K$ arms and time horizon $T$, if $\eta$ and $a$ are selected such that $\eta\psi(2\eta /a) \geq 2a$, then {Robust MOSS} satisfies
	\[ \supscr{R_T}{worst} \leq C  u K^{\frac{\epsilon}{1+\epsilon}} (T/e)^{\frac{1}{1+\epsilon}} + 2uK,\]
	where $C = \Gamma\left(1/\epsilon + 2\right) \left[a/\left(6+3\eta\right)\right]^{\frac{1}{\epsilon}} \left[{3}/{\psi \left (6+ 3\eta \right)} \right]^{\frac{1+\epsilon}{\epsilon}} + \epsilon \Gamma\left({1}/{\epsilon}+2\right) \left(6+3\eta\right)^{-\frac{1}{\epsilon}}  \big[{6a}/{\psi (2 \eta/a )} \big]^{\frac{1+\epsilon}{\epsilon}} a/\ln(a) + \left(6+3\eta\right) \big [e+(1+\epsilon) e^{\frac{-\epsilon}{1+\epsilon}}\big]$.
\end{theorem}

\begin{remark}
	Parameter $a$ and $\eta$ as inputs to Robust MOSS can be selected by minimizing the leading constant $C$ in the upper bound on the regret in Theorem~\ref{theorem:minimax bound}. We have found that selecting $a$ slightly larger than $1$ and selecting smallest $\eta$ that satisfies $\eta\psi(2\eta /a) \geq 2a$ yields good performance.
\end{remark}

\begin{proof}
	Since both the UCB and the regret scales with $u$ defined in Assumption~\ref{ass1}, to simplify the expressions, we assume $u=1$.  Also notice that Assumption~\ref{ass1} indicates $\abs{\mu_k}\leq u$, so $\Delta_k\leq 2$ for any $k \in \until{K}$. In the following, any terms with superscript or subscript ``$*$" and ``$k$" are with respect to the best and the $k$-th arm, respectively. The proof is divided into $4$ steps.  
	
	\noindent \textbf{Step 1:} We follow a decoupling technique inspired by the proof of regret upper bound in MOSS~\cite{MOSS}. Take the set of $\delta$-bad arms as $\mathcal{B}_\delta$ as
	\begin{equation} \label{def: badarm}
	\mathcal{B}_{\delta} := \setdef{k \in \until{K}}{\Delta_k > \delta},
	\end{equation}
	where we assign $\delta =  \left(6+3\eta\right)\left(e K/T\right)^{\frac{\epsilon}{1+\epsilon}}$. Thus,
	\begin{align}
	R_T &\leq T \delta + \sum_{t=1}^{K} \Delta_k +\expt \Bigg[\sum_{t=K+1}^{T} \indicator {\{\varphi_t \in \mathcal{B}_\delta \}} \left(\Delta_{\varphi_t} - \delta\right) \Bigg] \nonumber \\
	&\leq \! T \delta + 2K + \expt \Bigg[\sum_{t=K+1}^{T} \indicator {\{\varphi_t \in \mathcal{B}_\delta \}} \left(\Delta_{\varphi_t} - \delta\right) \!\Bigg]. \label{sum: regret}
	\end{align}
	Furthermore, we make the following decomposition 
	\begin{align}
	& \sum_{t=K+1}^T \mathbf{1}{\{\varphi_t \in \mathcal{B}_{\delta} \}} \left(\Delta_{\varphi_t}-\delta\right) \nonumber \\
	{=} & \sum_{t = K+1}^T  \indicator {\left\{\varphi_t \in \mathcal{B}_{\delta},  g_{n^* (t)}^* \leq \mu^* - \frac{\Delta_{\varphi_t}}{3} \right \}} \left(\Delta_{\varphi_t}-\delta\right) \label{error: under estimate}\\
	& +\sum_{t=K+1}^T  \indicator{\left\{\varphi_t \in \mathcal{B}_{\delta}, g_{n^* (t)}^* > \mu^* - \frac{\Delta_{\varphi_t}}{3} \right \}} \left(\Delta_{\varphi_t}-\delta\right).\nonumber
	\end{align}
	Notice that~\eqref{error: under estimate} describes regret from underestimating optimal arm $*$. For the second summand, since $g^{\varphi_t}_{ n_{\varphi_t} (t)} \geq g_{n^* (t)}^*$,
	\begin{align}
	&\sum_{t = K+1}^T \indicator {\left\{\varphi_t \in \mathcal{B}_{\delta},  g_{n^* (t)}^* > \mu^* - \frac{\Delta_{\varphi_t}}{3} \right \}} \left(\Delta_{\varphi_t}-\delta\right) \nonumber \\
	\leq & \sum_{t = K+1}^T \indicator {\left\{\varphi_t \in \mathcal{B}_{\delta}, g^{\varphi_t}_{n_{\varphi_t}(t)} > \mu_{\varphi_t} + \frac{2\Delta_{\varphi_t}}{3} \right \}} \Delta_{\varphi_t} \nonumber \\
	= & \sum_{k \in \mathcal{B}_\delta} \sum_{t = K+1}^T  \mathbf{1}{\left\{\varphi_t=k, g^k_{n_k(t)} > \mu_k + \frac{2\Delta_{k}}{3} \right\}} \Delta_k, \label{error: over estimate}
	\end{align}
	which characterizes the regret caused by overestimating $\delta$-bad arms.
	
	\noindent \textbf{Step 2:}	
	In this step, we bound the expectation of~\eqref{error: under estimate}. When event $\left\{\varphi_t \in \mathcal{B}_{\delta}, g_{n^* (t)}^* \leq \mu^* -  \Delta_{\varphi_t}/3 \right\}$ happens, we know
	\[\Delta_{\varphi} \leq  3\mu^* - 3g^*_{n^*(t)} \text{ and }  g^*_{n^*(t)} < \mu^* - \frac{\delta}{3}. \]	
	Thus, we get
	\begin{align*}
	&\mathbf{1}{\left\{\varphi_t \in \mathcal{B}_{\delta}, g_{n^* (t)}^* \leq \mu^* - \frac{\Delta_{\varphi_t}}{3} \right \}} (\Delta_{\varphi_t}-\delta) \\
	\leq &\indicator {\left\{ g_{n^*(t)}^* < \mu^* - \frac{\delta}{3} \right \}} \times  \big(3\mu^* - 3 g_{n^*(t)}^* - \delta \big):=Y_t.
	\end{align*}
	Since $Y_t$ is a positive random variable, its expected value can be computed involving only its cumulative density function:
	\begin{align*}
	\expt  \left[Y_t\right] & = \int_{0}^{+\infty} \prob  \left( Y_t>x \right) dx \\
	& \leq \int_{0}^{+\infty} \prob \left( 3\mu^* - 3 g_{n^*(t)}^* -\delta > x \right) dx \\
	&= \int_{\delta}^{+\infty} \prob \Big( \mu^* - g_{n^*(t)}^* > \frac{x}{3} \Big) dx.
	\end{align*}
	Then we apply Lemma~\ref{lemma: suff_sample} at optimal arm $*$ to get
	\[\expt[Y_t] \leq \frac{K C_1}{T} \int_{\delta}^{+\infty} \frac{1}{\epsilon} x^{-\frac{1+\epsilon}{\epsilon}} dx = \frac{K C_1}{T\delta^{\frac{1}{\epsilon}}}, \]
	where $C_1 = \epsilon \Gamma\left({1}/{\epsilon}+2\right) \big[{6a}/{\psi  (2 \eta /a )} \big]^{\frac{1+\epsilon}{\epsilon}} a/\ln(a)$. We conclude this step by
	\[\expt[\eqref{error: under estimate}]\leq \sum_{t = K+1}^T Y_t \leq C_1 K\delta^{-\frac{1}{\epsilon}}. \]
	
	\noindent \textbf{Step 3:}
	In this step, we bound the expectation of~\eqref{error: over estimate}. For each arm $k \in \mathcal{B}_\delta$,
	\begin{align}
	& \sum_{t = K+1}^T \mathbf{1}{ \left\{\varphi_t=k, g_{n_{k}(t)}^k \geq \mu_k + \frac{2\Delta_{k}}{3} \right\}} \nonumber \\
	= & \sum_{t = K+1}^T \sum_{m=1}^{t-K} \indicator \left\{\varphi_t=k, n_{k}(t) = m\right \} \indicator \left\{ g_m^k \geq \mu_k + \frac{2\Delta_k}{3} \right\} \nonumber \\
	= & \sum_{m = 1}^{T-K}  \indicator \left\{ g_m^k \geq \mu_k + \frac{2\Delta_k}{3} \right\} \! \sum_{t=m+K}^{T} \!\!\indicator \left\{\varphi_t=k, n_{k}(t) = m\right \} \nonumber \\
	\leq & \sum_{m = 1}^T  \indicator \left\{ g_m^k \geq \mu_k + \frac{2\Delta_k}{3} \right\} \nonumber \\
	\leq & \sum_{m=1}^{T} \indicator \Bigg\{ \frac{1}{m} \sum_{i=1}^{m} d_i^k \geq \frac{2\Delta_k}{3} - (2+\eta) c_m \Bigg\}, \label{ineq: 6}
	\end{align}	
	where in the last inequality we apply Lemma~\ref{lemma: estimator_error} and use the fact that ${u^{1+\epsilon}}/{B_m^\epsilon} \leq  c_m$ in~\eqref{ineq: bc}. 
	We set 
	\[l_k = \ceil{\left(\frac{6+3\eta}{\Delta_k}\right)^{\frac{1+\epsilon}{\epsilon}} \ln \Bigg (\frac{T}{K} \left(\frac{\Delta_k}{6+3\eta}\right)^{\frac{1+\epsilon}{\epsilon}} \Bigg )}.\]
	With $\Delta_k \geq \delta$, we get $l_k$ is no less than
	\[\left(\frac{6+3\eta}{\Delta_k}\right)^{\frac{1+\epsilon}{\epsilon}} \ln \bigg (\frac{T}{K} \left(\frac{\delta}{6+3\eta}\right)^{\frac{1+\epsilon}{\epsilon}} \bigg )  =\left(\frac{6+3\eta}{\Delta_k}\right)^{\frac{1+\epsilon}{\epsilon}}. \]
	Furthermore, since $c_m$ is monotonically decreasing with $m$, for $m \geq  l_k$,
	\begin{equation}
	c_m \leq c_{l_k} \leq \Bigg[\frac{\ln_+ \big(\frac{T}{K} \big(\frac{\Delta_k}{6+3\eta}\big)^{\frac{1+\epsilon}{\epsilon}}\big)}{l_k}\Bigg]^{\frac{\epsilon}{1+\epsilon}} \leq \frac{\Delta_k}{6+3\eta}. \label{ineq: 8}
	\end{equation}
	With this result and $l_k \geq 1$, we continue from~\eqref{ineq: 6} to get	
	\begin{align}
	\expt[\eqref{ineq: 6}]\leq &l_k-1 + \sum_{m = l_k}^T \prob \Bigg\{ \frac{1}{m} \sum_{i=1}^{m} d_i^k \geq \frac{2\Delta_k}{3}- (2+\eta) c_m  \Bigg\} \nonumber \\
	\leq & l_k-1 + \sum_{m = l_k}^T \prob \Bigg\{ \frac{1}{m} \sum_{i=1}^{m} d_i^k \geq \frac{\Delta_k}{3} \Bigg\}.	\label{ineq: 7}
	\end{align}
	Therefore by using Lemma~\ref{max_inq_b} together with statement (ii) from Lemma~\ref{prop: d_i}, we get
	\begin{align}
	&\sum_{m = l_k}^T \prob \Bigg\{ \frac{1}{m} \sum_{i=1}^{m} d_i^k \geq \frac{\Delta_k}{3} \Bigg\} \nonumber \\
	\leq & \sum_{m = l_k}^T \exp \left(-\frac{m \Delta_k}{3 B_m} \psi\left(B_m^\epsilon \Delta_k \right) \right) \nonumber\\
	\leq & \sum_{m = l_k}^T \exp \left(-\frac{m \Delta_k}{3 B_m} \psi\left(6+3\eta \right) \right), \label{newadd}
	\end{align}
	where the last step is due to that $\psi(x)$ is monotonically increasing and $B_m^\epsilon \Delta_k \geq (6+3\eta) B_m^\epsilon c_m  \geq 6+3\eta$ from~\eqref{ineq: 8} and~\eqref{ineq: bc}. Since $B_m =  \phi \big(h(m)\big)^{-\frac{1}{1+\epsilon}} \leq \phi(am)^{-\frac{1}{1+\epsilon}}\leq (am)^{\frac{1}{1+\epsilon}}$, we have
	\begin{align*}
	\eqref{newadd} &\leq \sum_{m = 1}^T \exp \left(-m^{\frac{\epsilon}{1+\epsilon}} a^{-\frac{1}{1+\epsilon}} \psi\left(6+3\eta \right) \frac{\Delta_k }{3} \right). \\
	&\leq  \int_{0}^{+\infty} \exp \left(-\beta y^{\frac{\epsilon}{1+\epsilon}}\right) dy \\	
	&\big(\text{where we set } \beta = a^{-\frac{1}{1+\epsilon}} \psi \left ( 6+3\eta \right) \Delta_k /3\big)\\	
	&=  \frac{1+\epsilon}{\epsilon} \beta^{-\frac{1+\epsilon}{\epsilon}} \int_{0}^{+\infty} z^{\frac{1+\epsilon}{\epsilon}-1} \exp \left(-z\right) dy \\
	&\big(\text{where }z = \beta y^{\frac{\epsilon}{1+\epsilon}}\big)\\
	&=  \Gamma\left(\frac{1}{\epsilon}+2\right)  \beta^{-\frac{1+\epsilon}{\epsilon}}.
	\end{align*}
	Plugging it into~\eqref{ineq: 7},
	\begin{align*}
	\expt[\eqref{ineq: 6}] & \leq C_2 \Delta_k^{-\frac{1+\epsilon}{\epsilon}} + C_3\Delta_k^{-\frac{1+\epsilon}{\epsilon}}  \ln \Big (\frac{T}{K C_3} \Delta_k^{\frac{1+\epsilon}{\epsilon}} \Big ),
	\end{align*}
	where $C_2 = \Gamma\left({1}/{\epsilon}+2\right) a^{\frac{1}{\epsilon}} \left[{3}/{\psi \left (6+ 3\eta \right)} \right]^{\frac{1+\epsilon}{\epsilon}} $ and $C_3 = \left(6+3\eta\right)^{\frac{1+\epsilon}{\epsilon}} $.
	Put it together with $\Delta_k \geq \delta$ for all $k \in \mathcal{B}_\delta$,
	\begin{align*}
	\expt[\eqref{error: over estimate}] &\leq \sum_{k \in \mathcal{B}_\delta} C_2 \Delta_k^{-\frac{1}{\epsilon}} + C_3\Delta_k^{-\frac{1}{\epsilon}} \ln \left (\frac{T}{K C_3} \Delta_k^{\frac{1+\epsilon}{\epsilon}} \right )\\
	& \leq C_2 K \delta^{-\frac{1}{\epsilon}} + (1+\epsilon) e^{\frac{-\epsilon}{1+\epsilon}} C_3 K \delta^{-\frac{1}{\epsilon}},
	\end{align*}
	where we use the fact that $x^{-\frac{1}{\epsilon}} \ln \big (Tx^{\frac{1+\epsilon}{\epsilon}} /\left(K C_3\right) \big )$ takes its maximum at $x = \delta\exp(\epsilon^2/(1+\epsilon))$.
	
	\noindent \textbf{Step 4:} 
	Plugging the results in step $2$ and step $3$ into~\eqref{sum: regret},
	\[\supscr{R_T}{worst} \leq T\delta + \left[C_1 + C_2 + (1+\epsilon) e^{\frac{-\epsilon}{1+\epsilon}} C_3\right]K\delta^{-\frac{1}{\epsilon}}+2K.\]
	Straightforward calculation concludes the proof.
\end{proof}

\subsection{Distribution-dependent Regret Upper Bound}
We now show that robust MOSS also preserves a logarithm upper bound on the {distribution-dependent} regret.
\begin{theorem}\label{theorem:distribution-dependent bound}
	For the heavy-tailed stochastic MAB problem with $K$ arms and time horizon $T$, if $\eta\psi(2\eta/a) \geq 2a$, the regret $R_T$ for {Robust MOSS} is no greater than
	\[ \sum_{k: \Delta_k > 0 } \Big(\frac{u^{1+\epsilon}}{\Delta_k}\Big)^{\frac{1}{\epsilon}} \left[C_1\ln \bigg (\frac{T}{KC_1} \Big(\frac{\Delta_k}{u}\Big)^{\frac{1+\epsilon}{\epsilon}} \bigg ) + C_2 K \right] + \Delta_k. \]
	where  $C_1 = \left(4+4\eta\right)^{\frac{1+\epsilon}{\epsilon}}$ and $C_2= \max \Big(e C_1, 2 \Gamma({1/\epsilon+2})  \left({8a}/{\psi  ( 2\eta/a )}\right)^{\frac{1+\epsilon}{\epsilon}} {a}/{\ln(a)}\Big) $.
\end{theorem}
\begin{proof}
	Let $\delta = \left(4+4\eta\right)\left(e K/T\right)^{\frac{\epsilon}{1+\epsilon}}$ and define $\mathcal{B}_\delta $ the same as~\eqref{def: badarm}. Since $\Delta_k \leq \delta$ for all $k \notin \mathcal{B}_\delta$, the regret satisfies
	\begin{align}		
	R_T &\leq \sum_{k \notin \mathcal{B}_\delta} T \Delta_k + \sum_{t=1}^{T} \indicator {\{\varphi_t \in \mathcal{B}_\delta \}} \Delta_{\varphi_t} \nonumber \\
	&\leq \sum_{k \notin \mathcal{B}_\delta} \! e K\left(\frac{4+4\eta}{\Delta_k}\right)^{\frac{1+\epsilon}{\epsilon}} \!\! \Delta_k  + \sum_{k\in \mathcal{B}_\delta}\sum_{t=1}^T \indicator{\{\varphi_t = k \}} \Delta_{k}. \label{bound2: 1}
	\end{align}
	Pick arbitrary $l_k \in \integer_{>0}$, thus
	\begin{align*}
	\sum_{t=1}^T \indicator \{\varphi_t = k \} & \leq l_k + \!\!\!\sum_{t = K+1}^T \!\! \indicator\left\{ \varphi_t = k, n_k(t)\geq l_k \right \} \\
	&\leq l_k + \!\!\! \sum_{t=K+1}^T \!\! \indicator{\left\{g^k_{n_{k}(t)} \geq g_{n^* (t)}^*, n_k(t)\geq l_k\right \}}.
	\end{align*}
	Observe that $g^k_{n_{k}(t)} \geq g_{n^* (t)}^*$ implies at least one of the following is true
	\begin{align}
	&g_{n^* (t)}^* \leq \mu^* -{\Delta_k}/{4}, \label{event:1}\\
	&g^k_{n_{k}(t)} \geq \mu_{k} + {\Delta_{k}}/{4} + 2(1+\eta) c_{n_k(t)}, \label{event:2}\\
	&(1+\eta) c_{n_k(t)} > {\Delta_k}/{4}. \label{event:3}
	\end{align}
	We select
	\[l_k = \ceil{\left(\frac{4+4\eta}{\Delta_k}\right)^{\frac{1+\epsilon}{\epsilon}} \ln \Bigg(\frac{T}{K} \left(\frac{\Delta_k}{4+4\eta}\right)^{\frac{1+\epsilon}{\epsilon}} \Bigg )}.\]
	Similarly as~\eqref{ineq: 8}, $n_k(t) \geq l_k$ indicates $c_{n_k(t)} \leq \Delta_k/(4+4\eta)$, so~\eqref{event:3} is false. Then we apply Lemma~\ref{lemma: suff_sample} and Corollary~\ref{corollary: suff_sample},
	\begin{align*}
	&\prob {\left\{g^k_{n_{k}(t)} \geq g_{n^* (t)}^*, n_k(t)\geq l_k\right \}}\\
	\leq & \prob \left( \text{\eqref{event:1} or~\eqref{event:2} is true } \right) \leq \frac{C_2' K}{T }\Delta_k^{-\frac{1+\epsilon}{\epsilon}},
	\end{align*}
	where $C_2'= 2 \Gamma\left({1/\epsilon+2}\right)  \left({8a}/{\psi  (2\eta / a )}\right)^{\frac{1+\epsilon}{\epsilon}} {a}/{\ln(a)}$. Substituting it into~\eqref{bound2: 1}, $R_T$ is upper bounded by
	\begin{align*}
	& \sum_{k \notin \mathcal{B}_\delta} \! \frac{e C_1 K }{\Delta_k^{\frac{1}{\epsilon}}} + \! \sum_{k \in \mathcal{B}_\delta} \!  \left[\frac{C_1 }{\Delta_k^{\frac{1}{\epsilon}}}\ln \bigg (\frac{T}{K C_1} \Delta_k^{\frac{1+\epsilon}{\epsilon}} \bigg ) + \frac{C_2' K}{\Delta_k^{\frac{1}{\epsilon}}} + \Delta_k \right]. 
	\end{align*}
	Considering the scaling factor $u$, the proof can be concluded with easy computation.
\end{proof}

\section{Numerical Illustration}\label{sec:simulation}
{In this section, we compare Robust MOSS with MOSS and Robust UCB (with truncated empirical mean or Catoni's estimator)~\cite{bubeck2013bandits} in a $3$-armed heavy-tailed bandit setting.} The mean rewards are $\mu_1=-0.3$, $\mu_2=0$ and $\mu_3=0.3$ and sampling at each arm $k$ returns a random reward equals to $\mu_k$ added by sampling noise $\nu$, where $\abs{\nu}$ is a generalized Pareto random variable and the sign of $\nu$ has equal probability to be positive and negative. The PDF of reward at arm $k$ is
\[f_{k}(x) = \frac{1}{2\sigma}\left(1 + \frac{\xi \abs{x-\mu_k}}{\sigma}\right)^{-\frac{1}{\xi} - 1} \, \text{for } x \in (-\infty, +\infty),\]
where we select $\xi = 0.33$ and $\sigma=0.32$. Thus, for a random reward $X$ from any arm, we know $\expt [X^2] \leq 1$, which means $\epsilon=1$ and $u=1$. We select parameters $a=1.1$ and $\eta = 2.2$ for Robust MOSS so that condition $\eta\psi(2\eta/a) \geq 2a$ is met. 

Fig.\ref{fig.comparison} shows the mean cumulative regret together with quantiles of cumulative regret distribution as a function of time, which are computed using $200$ simulations of each policy. The simulation result shows that there is a chance MOSS loses stability in heavy-tailed MAB and suffers linear cumulative regret while other algorithms work consistently and maintain sub-linear cumulative regrets. {Robust MOSS slightly outperforms Robust UCB in this specific problem.}



\begin{figure}[ht!]
	\begin{center}
		\includegraphics[width=0.49\textwidth]{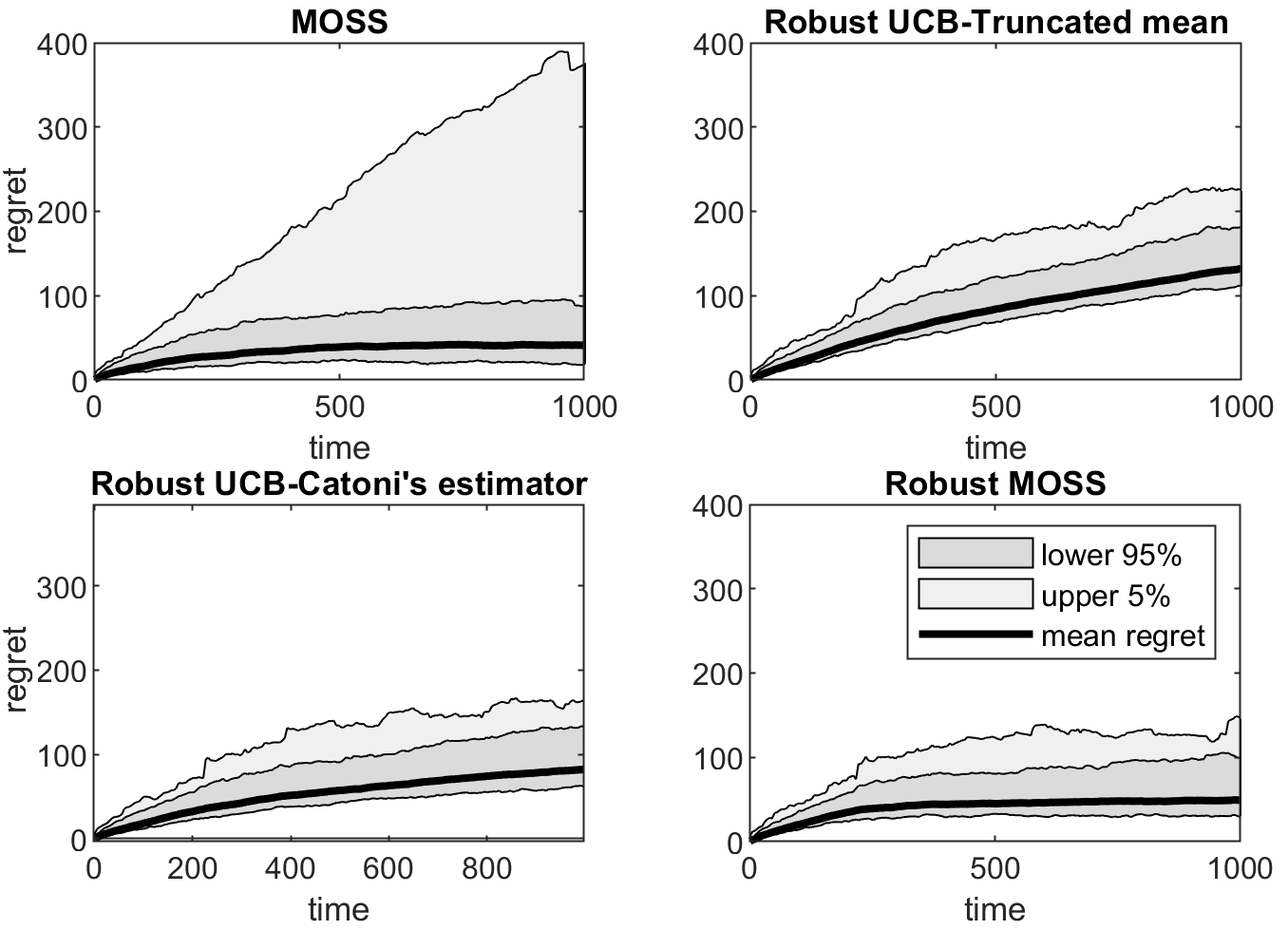}
		\vspace{-16pt}
		\caption{Comparison of $4$ algorithms in heavy-tailed MAB: On each graph, the bold curve is the mean regret while light shaded and dark shaded regions correspond respectively to upper $5\%$ and lower $95\%$ quantile cumulative regrets.}
		\vspace{-10pt}
		\label{fig.comparison}		
	\end{center}
\end{figure}

\section{Conclusions and Future Direction}\label{sec:conclusions}
We proposed the Robust MOSS algorithm for heavy-tailed bandit problem. We evaluated the algorithm by deriving upper bounds on the associated distribution-free and distribution-dependent regrets. Our analysis showed that Robust MOSS achieves order optimal performance in both scenarios.
The saturated mean estimator centers at zero which make the algorithm not translation invariant. Exploration of translation invariant robust mean estimator in this context remains an open problem. 

\bibliographystyle{IEEEtran}
\bibliography{IEEEabrv,bandits,surveillance,mybib}
	
\end{document}